\documentclass[hidelinks,letter, 10 pt, conference,doublecolomn]{ieeeconf} 
\IEEEoverridecommandlockouts    
\usepackage{balance}
\usepackage{mathtools}

\title{\Large \bf  Stability Analysis of Geometric Control for a Canonical Class of Underactuated Aerial Vehicles with Spurious Forces}
\usepackage{mathrsfs} 
\usepackage{bm}
\usepackage{subfiles}
\usepackage{graphicx}
\usepackage{epstopdf}
\usepackage{float}
\usepackage{bm}

\usepackage{bm}
\usepackage{amsmath}
\usepackage{amssymb}

\newcommand{\vect}[1]{\mathbf{#1}}
\newcommand{\mat}[1]{\mathbf{#1}}

\newcommand{\diffs}[3]{\frac{\partial^2 #1}{
\ifx#2#3 
\partial #2^2
\else
\partial #2 \partial #3
\fi
}}

\newcommand{\av}{\vect{a}}

\newcommand{\dv}{\vect{d}}
\newcommand{\ev}{\vect{e}}

\newcommand{\fv}{\vect{f}}

\newcommand{\kv}{\vect{k}}

\newcommand{\pv}{\vect{p}}

\newcommand{\rv}{{\vect{r}}}

\newcommand{\uv}{\vect{u}}
\newcommand{\vv}{\vect{v}}

\newcommand{\wv}{\vect{w}}

\newcommand{\xv}{\vect{x}}

\newcommand{\yv}{\vect{y}}

\newcommand{\zv}{\vect{z}}

\newcommand{\SO}{\mathbf{SO}(3)}
\newcommand{\SE}{\mathbf{SE}(3)}

\newcommand{\Omegav}{\bm{\Omega}}

\newcommand{\rank}{\operatorname{rank}}

\newcommand{\tauv}{\bm{\tau}}

\newcommand{\omegav}{\bm{\omega}}

\newcommand{\Am}{\mat{A}}
\newcommand{\Bm}{\mat{B}}
\newcommand{\Cm}{\mat{C}}

\newcommand{\Dm}{\mat{D}}

\newcommand{\Jm}{\mat{J}}

\newcommand{\Km}{\mat{K}}
\newcommand{\Mm}{\mat{M}}

\newcommand{\Pm}{\mat{P}}

\newcommand{\Rm}{\mat{R}}

\newcommand{\Wm}{\mat{W}}
\newcommand{\Xm}{\mat{X}}

\usepackage{svg}
\usepackage{hyperref}
\usepackage{academicons}
\usepackage{xcolor}

\usepackage{amsthm}
\theoremstyle{plain}

\newtheorem{prop}{Proposition}
\newtheorem{thm}{Theorem}
\newtheorem{pf}{Proof}
\newtheorem{lem}{Lemma}[section]
\newtheorem{defn}{Definition}
\newtheorem{problem}{Problem}

\usepackage{comment}
\usepackage[font=small]{caption}
\usepackage{subcaption}
\usepackage{calc}
\newtheorem{rem}{\textbf{Remark}}[section]
\usepackage{mathtools} 
 
\usepackage{mathptmx}
\usepackage{orcidlink}
 
\usepackage{tikz}
\usepackage{tikz-3dplot}
\usetikzlibrary{arrows.meta}
 
 \usepackage{tabularx}
\author{Simone Orelli$^{2, \orcidlink{0009-0007-6613-8597}}$, Mirko Mizzoni$^{1,\orcidlink{0009-0006-2165-3475}}$,\IEEEmembership{Student Member, IEEE} and Antonio Franchi$^{1,2,\orcidlink{0000-0002-5670-1282}}$,\IEEEmembership{Fellow, IEEE}
\thanks{$^1$Robotics and Mechatronics group, Faculty of Electrical Engineering,  Mathematics, and Computer Science (EEMCS), University of Twente, 7500 AE Enschede, The Netherlands. {\footnotesize \tt m.mizzoni@utwente.nl}, {\footnotesize } {\footnotesize \tt schol@r-franchi.eu}}
\thanks{$^2$Department of Computer, Control and Management Engineering, Sapienza University of Rome, 00185 Rome, Italy, {\footnotesize \tt s.orelli@uniroma1.it, schol@r-franchi.eu } {\footnotesize \tt }}\thanks{This work was partially funded by the Horizon Europe research agreement no. 101120732 (AUTOASSESS).}}

\begin{document}
\newif\ifappendix
\appendixtrue 
 
\newcommand{\af}[1]{\textcolor{blue}{\textbf{AF:} #1}}
 
\maketitle


\begin{abstract}
Standard geometric control relies on force--moment decoupling, an assumption that breaks down in many aerial platforms due to \emph{spurious forces} naturally induced by control moments. While strategies for such coupled systems have been validated experimentally, a rigorous theoretical certification of their stability is currently missing. This work fills this gap by providing the first formal stability analysis for a generic class of floating rigid bodies subject to spurious forces. We introduce a canonical model and construct a Lyapunov-based proof establishing local exponential stability of the hovering equilibrium. Crucially, the analysis explicitly addresses the structural challenges--specifically the induced non-minimum-phase behavior--that prevent the application of standard cascade arguments.
\end{abstract}

\section{Introduction}
Unmanned Aerial Vehicles (UAVs) are becoming increasingly pervasive in everyday life, thanks to their versatility, agility, and rapidly decreasing cost. They are now routinely employed in a wide range of applications, including surveillance, aerial photography, inspection of civil infrastructures, search-and-rescue operations, and aerial manipulation~\cite{2018-KimMokBen,2016-RaoGopMai,2020-HayYanBetBro,2019e-TogTelGasSabBicMalLanSanRevCorFra}.  
The growing deployment of UAVs in safety-critical and autonomous missions has consequently intensified the demand for reliable and high-performance control strategies.

Over the years, several control paradigms have been proposed to address these objectives. Classical nonlinear control techniques, including dynamic feedback linearization and backstepping, have been extensively investigated for UAV stabilization and trajectory tracking~\cite{2022-MohLawJiaLyn,2006-MadBen,2007-BouSie}. More recently, geometric control laws defined directly on nonlinear manifolds have attracted considerable attention due to their coordinate-free formulation and rigorous stability guarantees~\cite{2020-LiuHudSunYu,2019-HeWanLiu,2010-LeeLeoMcc,2013-LeeSreKum,2015-GooLeeLee}. 

A key property exploited by most existing UAV control laws is the intrinsic structural decoupling between translational and rotational dynamics. Under nominal conditions, control moments affect the vehicle attitude, while translational motion is governed by the thrust force expressed in the inertial frame through the vehicle orientation. This separation greatly simplifies both control design and stability analysis and underpins the effectiveness of many state-of-the-art controllers. However, this structural property is not guaranteed to hold in all scenarios. When the decoupling between rotational and translational dynamics is compromised, additional force components appear in the translational dynamics. These \emph{spurious forces} arise as a direct consequence of allocating control moments and are therefore intrinsically linked to the system dynamics, rather than being external disturbances or modeling uncertainties.

\begin{figure}[htbp]
    \centering

    \tdplotsetmaincoords{65}{-25} 

    \begin{tikzpicture}[tdplot_main_coords, scale=1.1]

        \tikzset{
            axisW/.style={
                -{Latex[length=3pt,width=3pt]},
                line width=0.6pt
            },
            axisB/.style={
                -{Latex[length=3pt,width=3pt]},
                line width=0.6pt
            },
            body/.style={
                fill=black!5,
                draw=black!70
            },
            cleanForce/.style={
                -{Latex[length=4pt,width=4pt]},
                line width=0.9pt
            },
            spuriousForce/.style={
                -{Latex[length=4pt,width=4pt]},
                line width=0.9pt,
                red!70!black
            },
            totalForce/.style={
                -{Latex[length=5pt,width=5pt]},
                line width=1.1pt
            },
            helperLine/.style={
                dashed,
                line width=0.4pt,
                gray!60
            },
            frameLabel/.style={
                font=\scriptsize
            },
            vecLabel/.style={
                font=\scriptsize,
                inner sep=1pt,
                fill=white,
                draw=none
            }
        }

        \coordinate (B) at (2,0,-1); 

        \shade[body] (B) ellipse [x radius=1.1, y radius=0.55];

        \coordinate (Top) at (2,0,-0.65);


        \coordinate (W) at (-3.5,-1,-0.25);
        \draw[axisW] (W) -- ++(1,0,0) node[below right,frameLabel] {$\mathbf{x}_W$};
        \draw[axisW] (W) -- ++(0,1,0) node[left,frameLabel] {$\mathbf{y}_W$};
        \draw[axisW] (W) -- ++(0,0,1) node[above,frameLabel] {$\mathbf{z}_W$};
        \node[frameLabel, below=1pt] at (W) {$\mathcal{F}_W$};

        \draw[helperLine,->] (W) -- (B)
            node[midway, vecLabel, yshift=-5pt] {$\pv$};

        \node[circle, inner sep=1pt, fill=black] at (B) {};

        \draw[axisB] (B) -- ++(0.9848,0,0.1736) node[ right,frameLabel] {$\mathbf{x}_B$};
        \draw[axisB] (B) -- ++(0,1,0)           node[left,frameLabel]        {$\mathbf{y}_B$};
        \draw[axisB] (B) -- ++(-0.1736,0,0.9848) node[left,frameLabel]      {$\mathbf{z}_B$};
        \node[frameLabel, below left=0pt] at (2,0,-0.9) {$\mathcal{F}_B$};


        \coordinate (dstarEnd) at (2,-1.6,2.3);
        \draw[helperLine] (B) -- (dstarEnd)
            node[midway, vecLabel, xshift=0pt, yshift=22pt] {$\mathbf d_\star$};

        \coordinate (fCleanEnd) at (2,-1.4,1);
        \draw[cleanForce] (B) -- (fCleanEnd)
            node[midway, vecLabel, xshift=-5pt, yshift=20pt] {$\Am \mathbf u_f$};

        \coordinate (fEnd) at (4,1.5,0.1);
        \draw[totalForce] (B) -- (fEnd)
            node[midway, vecLabel, xshift=26pt, yshift=20pt] {$\mathbf f$};

        \draw[spuriousForce] (fCleanEnd) -- (fEnd)
            node[midway, vecLabel, xshift = 5pt, yshift=12pt] {$\Bm\,\mathbf u_\tau$};

        \draw[->, thick]
            (1.3,-1,-1)
            to[bend left=35]
            (1,1.3,-1.2)
            node[midway, vecLabel, xshift=20pt, yshift=-30pt] {$\boldsymbol{\tau}$};

    \end{tikzpicture}

    \caption{Floating rigid body with frames $\mathcal{F}_W$ and $\mathcal{F}_B$, desired direction $\mathbf d_\star$, total force $\mathbf f$, and its decomposition into clean ($\Am \mathbf u_f$) and spurious ($\Bm \mathbf u_\tau$) components induced by the moment $\boldsymbol{\tau}$.}
    \label{fig:spurious_force}
\end{figure}
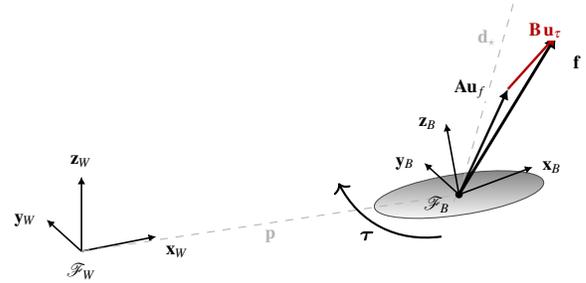 

Hovering plays a central role in this context. It may be required either deliberately, as an operational objective, or implicitly, as a necessary condition following the failure or degradation of one or more actuators. In such fault scenarios, the vehicle may lose its nominal decoupling properties, causing spurious forces to emerge even if the original platform was fully decoupled under healthy conditions. These situations pose a significant challenge in aerial robotics. The presence of intrinsic translational--rotational coupling fundamentally alters the system behavior and undermines key assumptions underlying conventional control designs. Developing control strategies capable of guaranteeing stable hovering and safe operation under such conditions remains an open and relevant problem for modern UAV systems.


Table~\ref{tab:comparison} summarizes the existing literature on floating aerial vehicles affected by force--moment coupling, highlighting the scope and limitations of current approaches.
Early contributions addressed this coupling through approximate or dynamic feedback linearization, where the coupling terms are explicitly retained in the model but mitigated via dynamic extensions of the control input, as originally proposed for VTOL aircraft and helicopter systems by Hauser and co-authors and by Koo and Sastry~\cite{1992-HauSasMey,1998-KooSas}. 
More recent works have instead adopted geometric control formulations that explicitly accommodate force--moment coupling. Among them, Zhong \emph{et al.}~\cite{2024-ZhoCaiMaDinFoo} established almost global exponential stability for trajectory tracking of quadrotors with tilted propellers, covering configurations with force allocation matrices of rank one, two, and three. However, the resulting analysis remains inherently restricted to four-rotor platforms and does not naturally extend to vehicles with different actuator topologies or to more general classes of floating rigid bodies exhibiting structural force--moment coupling.

\begin{table*}[t]
    \centering
    \caption{Comparison of prior work on floating vehicles affected by structural force--moment coupling. The proposed approach establishes Lyapunov-based local exponential stability of the hovering equilibriucm for a canonical model that abstracts the actuator topology.}\label{tab:comparison}
    \renewcommand{\arraystretch}{2}
    \begin{tabularx}{\textwidth}{l X l X X c}
        \hline
        \textbf{Work} & \textbf{Platform \& Validity} & \textbf{Task} & \textbf{Control Approach} & \textbf{Stability Result} & \textbf{Proof} \\
        \hline

        Martin \emph{et al.} \newline \cite{1994-MarDevPad}
        &
        \textbf{Restricted to:} \newline Planar VTOL (2D only)
        &
        Output tracking
        &
        Flatness-based: exact linearization + inversion
        &
        Local exponential state tracking
        &
        \textbf{Yes}
        \\
        \hline

        Koo \& Sastry \newline \cite{1998-KooSas}
        &
        \textbf{Restricted to:} \newline Standard Helicopter \newline (Main + Tail rotor)
        &
        Output tracking
        &
        Approximate I/O linearization (coupling neglected)
        &
        Bounded tracking error (Exact linearization is unstable)
        &
        \textbf{Yes}
        \\
        \hline

        Zhong \emph{et al.} \newline \cite{2024-ZhoCaiMaDinFoo}
        &
        \textbf{Restricted to:} \newline Quadrotors with tilted propellers ($4$ actuators)
        &
        Output tracking
        &
        Geometric $\SE$ controller
        &
        Almost global exponential tracking
        &
        \textbf{Yes}
        \\
        \hline
        
        Michieletto \emph{et al.} \newline \cite{2020e-MicCenZacFra}
        &
        \textbf{Generic:} \newline $N$-rotor platforms with arbitrary arrangement
        &
        Static hovering
        &
        Hierarchical nonlinear control 
        &
        Local asymptotic stability (excluding some configurations)
        &
        \textbf{Yes}
        \\
        \hline

        Michieletto \emph{et al.} \newline \cite{2017f-MicRylFra}
        &
        \textbf{Generic:} \newline $N$-rotor platforms with arbitrary arrangement
        &
        Static hovering
        &
        Geometric $\SE$ controller
        &
        \emph{Empirical validation only} (Simulations \& Exp.)
        &
        No
        \\
        \hline\hline

        \textbf{This Work}
        &
        \textbf{Abstract Canonical Model:} \newline Any underactuated vehicle with structural coupling
        &
        Static hovering
        &
        Geometric $\SE$ controller
        &
        \textbf{Local exponential stability of the hovering equilibrium}
        &
        \textbf{Yes}
        \\
        \hline
    \end{tabularx}
\end{table*}

Distinct from platform-specific solutions, the control framework proposed by Michieletto \emph{et al.}~\cite{2017f-MicRylFra} addresses a significantly broader class of systems, referred to as \emph{generically tilted multi-rotors}, admitting an arbitrary number of actuators and diverse geometric configurations. The objective of this line of work is to enable static hovering even in the presence of propeller failures, including configurations in which translational and rotational dynamics are intrinsically coupled. The proposed approach departs from the classical assumption that the control force is orthogonal to the plane containing the propeller centers and that control moments can be generated independently of the thrust force. Instead, the authors consider the more general case in which the direction of the control force is not fixed and control moments may induce translational effects. Within this framework, they introduce the concept of \emph{zero-moment direction} (also referred to as \emph{preferential direction}), defined as a virtual direction along which the intensity of the control force can be freely assigned while maintaining a zero net control moment. While this formulation provides a potentially unifying control paradigm for a wide range of multirotor configurations, its stability properties are supported primarily by heuristic arguments. Nevertheless, extensive experimental validation has demonstrated a remarkably large basin of attraction in fault scenarios. In particular, the controller was shown to successfully recover stability despite transient attitude errors approaching $50^\circ$, angular rates exceeding $50^\circ/\mathrm{s}$, and velocity errors above $1\,\mathrm{m/s}$~\cite{2018a-MicRylFra}. 

A subsequent work~\cite{2020e-MicCenZacFra} addresses a similar problem by providing sufficient conditions for local asymptotic stability of a hovering equilibrium within the same general modeling framework. However, the analysis introduces additional feasibility requirements on the actuation capabilities that effectively restrict the result to configurations in which  the translational component induced by moment allocation can be locally eliminated by an equivalent input parametrization. Therefore, the stability guarantees of~\cite{2020e-MicCenZacFra} do not apply to generic platforms with spurious forces. 

Therefore, despite the strong empirical evidence provided in~\cite{2018a-MicRylFra}, and the theoretical contribution provided in~\cite{2020e-MicCenZacFra}, the stability analysis of a controller applicable to the class of floating rigid bodies with spurious forces (such as the one presented in~\cite{2017f-MicRylFra}) is still an open problem.

Our contribution provides, for the first time, a rigorous Lyapunov-based stability proof for the geometric control law proposed by~\cite{2017f-MicRylFra}, extending its validity to the abstract canonical model presented in Section~\ref{sect:prel}. Our analysis relaxes the feasibility requirements~\cite{2020e-MicCenZacFra}, thereby being applicable to the class of floating vehicles with spurious forces. This result closes the gap between prior empirical evidence~\cite{2017f-MicRylFra} and theoretical certification. We prove that the controller guarantees local exponential stability of the hovering equilibrium. 
This work focuses exclusively on the theoretical certification, serving as the formal counterpart to the experimental studies of~\cite{2017f-MicRylFra}.


From the methodological point of view, the proposed analysis
is a non-trivial extension of the methodology presented in Lee~\emph{et al.}~\cite{2010-LeeLeoMcc,2013-LeeSreKum}. The presence of spurious forces fundamentally alters the system structure compared to~\cite{2010-LeeLeoMcc}. In particular, the induced non-minimum-phase behavior and intrinsic coupling violate the skew-symmetry and cascade properties on which the analysis of~\cite{2010-LeeLeoMcc} relies. To overcome these structural obstacles, we introduce a reformulated Lyapunov argument with new bounds specifically designed to handle the additional coupling terms. 

The remainder of the work is organized as follows. Section~\ref{sect:prel} introduces the generic canonical model and summarizes the control law. Section~\ref{sect:proof} details the stability analysis. Concluding remarks are provided in Section~\ref{sect:conclusions} and the Appendix~\ref{appendix} collects some of the  identities used throughout the work.

\section{Canonical Model and Actuation Properties}\label{sect:prel}
This section introduces the canonical model for aerial rigid bodies and recalls key actuation properties. For further details, the interested reader is referred to \cite{2017f-MicRylFra}.

\subsection{Canonical Model of Aerial Rigid Bodies}
We consider a rigid body moving freely in three-dimensional space, actuated by $n$ independent control inputs. Two right-handed coordinate frames are defined to describe its motion. The inertial frame $\mathcal{F}_W$ has origin $\mathcal{O}_W$ and orthonormal basis $\{\xv_W,\yv_W,\zv_W\}$. The body-fixed frame $\mathcal{F}_B$ is attached to the rigid body's center of mass (CoM) $\mathcal{O}_B$ and has basis $\{\xv_B,\yv_B,\zv_B\}$. The configuration of the rigid body is represented by the pair $(\pv,\Rm)\in \SE$, where $\Rm\in\SO$ denotes the rotation matrix describing the orientation of $\mathcal{F}_B$ relative to $\mathcal{F}_W$, and $\pv \in \mathbb{R}^3$ denotes the position of the CoM expressed in the inertial frame. The angular velocity of $\mathcal{F}_B$ relative to $\mathcal{F}_W$, expressed in $\mathcal{F}_B$, is denoted by $\Omegav\in \mathbb{R}^3$.

Using the Newton-Euler approach and neglecting the second order effects, the dynamics of the multi-rotor vehicle is approximated by the following set of equations: 
\begin{align}\label{eq:model}
    \begin{cases}
        \dot{\pv} &= \vv,\\
        m\dot{\vv} &= -mg\ev_3 + \Rm\big(\Am \uv_f + \Bm \uv_\tau\big),\\
        \dot{\Rm} &= \Rm\,\Omegav^\times,\\
        \Jm\dot{\Omegav} &= -\Omegav^\times\Jm\Omegav + \Cm \uv_{\tau},
    \end{cases}
\end{align}
where $\ev_3$ is the third column of the identity matrix in $\mathbb
{R}^{3\times 3}$, $g>0$ is the gravitational acceleration, $m>0$  is the platform mass, and $\Jm \in \mathbb{R}^{3\times 3}$ is the constant inertia matrix expressed in $\mathcal{F}_B$. The operator $(\cdot)^{\times}$ denotes the map that associates any non-zero vector in $\mathbb{R}^3$ to the related skew symmetric matrix in the special orthogonal Lie Algebra $\mathfrak{so}(3)$. 

The control inputs are $\uv_f \in \mathcal{U}_f \subseteq \mathbb{R}^{n_f}$ and $\uv_\tau \in \mathcal{U}_\tau \subseteq \mathbb{R}^{n_\tau}$, with $n_f + n_\tau = n$. The former regulates the body-frame control force, while the latter regulates the control moment and may, in general, induce an associated force component. The corresponding mappings from inputs to body-frame wrenches are described by the allocation matrices $\Am \in \mathbb{R}^{3\times n_f}$, and $\Bm ,\Cm \in \mathbb{R}^{3\times n_\tau}$. 

The model~\eqref{eq:model} can be rewritten in the compact form:
\begin{equation}\label{eq:system}
\dot{\xv} = \fv(\xv, \uv),
\end{equation}
whose state is defined by the tuple $\xv = (\pv, \vv, \Rm, \Omegav)\in \mathbb{R}^3 \times \mathbb{R}^3 \times \SO \times \mathbb{R}^3$.

\subsection{Actuation Properties}
We restrict attention to any aerial vehicles satisfying the two structural assumptions:
\begin{enumerate}
    \item Full moment authority:
    \begin{equation}\label{eq:rankC}
        \rank(\Cm)=3
    \end{equation}
    This condition, met whenever $n_\tau \ge 3$, ensures that every desired moment $\tauv \in \mathbb{R}^3$ is attainable via $\uv_\tau$.
    \item The matrix $\Am$ satisfies the full-rank condition, i.e., 
    \begin{equation}\label{eq:rankA}
        \rank(\Am)=\min\{3,\,n_f\}=\min\{3,\,n-n_\tau\}
    \end{equation}
\end{enumerate}
Define the \emph{moment–free force subspace} as $\text{Im}(\Am)$, whose dimension $\rank(\Am)$ measures how many independent force directions can be produced without affecting the attitude dynamics. In contrast, $\Bm\uv_\tau$ in \eqref{eq:model} represents the \emph{spurious force}, i.e., the component of the force that is unavoidably coupled with moment commands.

This leads to the following classification of platforms:
\begin{itemize}
    \item \textbf{Fully Decoupled (FD)} if $\text{Im}(\Bm)\subseteq \text{Im}(\Am)$, equivalently $\rank\!\left(\begin{bmatrix}\Am & \Bm\end{bmatrix}\right)=\rank(\Am)$. In this case, every force component generated by $\uv_\tau$ lies in $\text{Im}(\Am)$ and can be exactly compensated by a suitable choice of $\uv_f$.
    \item \textbf{Partially Coupled (PC)} if $\text{Im}(\Bm)\not\subseteq \text{Im}(\Am)$, equivalently $\rank\!\left(\begin{bmatrix}\Am & \Bm\end{bmatrix}\right)>\rank(\Am)$. In this case, $\uv_\tau$ generates at least one force component outside $\text{Im}(\Am)$ that cannot be compensated by any $\uv_f$.
\end{itemize}

Moreover, we say that the platform
\begin{itemize}
    \item has a \textbf{decoupled direction (D1)} if $\rank(\Am)\ge1$ (i.e., $n_f \ge 1$);
    \item has a \textbf{decoupled plane (D2)} if $\rank(\Am)\ge2$ (i.e., $n_f \ge 2$);
    \item is \textbf{fully actuated (FA)} if $\rank(\Am)=3$ (i.e., $n_f \ge 3$).
\end{itemize}

If the platform admits a single decoupled direction, we call it the \emph{preferential direction} and denote it by $\dv_\star \in \mathbb{S}^2$. 
Analogously, if the platform admits a decoupled plane, we call it the \emph{preferential plane} and denote it by $\Dm_\star \in \mathbb{R}^{3\times 2}$, whose columns are unit–norm and span a two–dimensional subspace of $\mathbb{R}^3$. 
Formally, the preferential plane is described by any matrix $\Dm_\star$ such that
\begin{equation}\label{eq:preferential_plane}
    \Dm_\star \Dm_\star^\top \;=\; \Am \Am^\dagger,
\end{equation}
where $\Am \Am^\dagger$ is the orthogonal projector onto $\text{Im}(\Am)$. See Fig.~\ref{fig:spurious_force} for a geometric illustration of the preferential direction, and spurious force.

In this work, we focus on platforms with 
\begin{equation}\label{eq:assumpt}
    n_f = 1,   \quad n_{\tau}=3,
\end{equation}
 i.e., systems that can generate a control force along a single moment–free direction. 
This setting corresponds to the minimally actuated case with $n = 4$ inputs  and represents the PC--D1 class.

This choice is not restrictive, but rather targets the most constrained scenario: any platform with higher force authority (e.g., $n_f \ge 2$) naturally admits a subset of directions satisfying the D1 condition. Consequently, the stability guarantees derived for this minimal configuration remain valid for any system possessing \emph{at least} one decoupled direction, simply by selecting $\dv_\star \in \text{Im}(\Am)$.
In this specific context, the preferential plane degenerates to the single preferential direction $\dv_\star$, and $\Am$ reduces to a single nonzero column $\av\in\mathbb{R}^3$. With $\Am^\dagger=\av^\top/\|\av\|^2$, \eqref{eq:preferential_plane} yields
\begin{equation}
    \dv_\star \dv_\star^\top \;=\; \Am \Am^\dagger \;=\; \av \av^\top/\|\av\|^2,
\end{equation}
which uniquely identifies the unit vector in the body–fixed frame:
\begin{equation}\label{eq:preferential_direction}
    \dv_\star \;=\; \av/\|\av\|.
\end{equation}

\subsection{Static Hovering}
We focus on \emph{statically hoverable} platforms, defined as systems whose position and attitude can be maintained at a constant equilibrium, i.e.,
\begin{equation}
    \vv = \Omegav = \mathbf{0},\quad 
    \begin{bmatrix}\Am & \Bm \end{bmatrix}\uv = mg\,\Rm^\top\ev_3,\quad
    \Cm \uv_\tau = \mathbf{0},
\end{equation}
which is in general guaranteed only if certain attitudes are realized by the platform (\emph{static hoverability realizability}~\cite{2017f-MicRylFra}).
A statically hoverable platform must satisfy: i) $n \ge 4$; ii) condition \eqref{eq:rankC}; and iii) $\rank(\Am)\ge 1$, meaning that there exists at least one direction along which the platform can generate thrust independently of the control moment. It is worth verifying that the PC--D1 case considered here satisfies the conditions in \eqref{eq:assumpt}.

\subsection{Control Law}\label{sec:control_law} 
We now briefly summarize the hierarchical control strategy proposed in~\cite{2017f-MicRylFra}, which consists of three layers: a position controller, an attitude controller, and a wrench mapper. The control objective is to regulate the platform to a constant reference position $\pv_r$ and a constant orientation $\Rm_r$.

The position controller computes a reference control force $\fv_r \in \mathbb{R}^3$ as:
\begin{equation}\label{eq:f_r}
    \fv_r = mg \ev_3 -\Km_p \ev_p - \Km_v \vv,
\end{equation}
where $\ev_p \coloneqq \pv - \pv_r$ is the position tracking error and \mbox{$\Km_p, \Km_v \in \mathbb{R}^{3\times3}$} are positive definite matrices.

To ensure that the thrust direction aligns with $\fv_r$, a desired orientation $\Rm_d$ is computed such that:
\begin{equation}
    \Rm_d \dv_\star = \fv_r/\|\fv_r\|.
\end{equation}
Specifically, we define $\Rm_d = \Rm_w \Rm_b$, where $\Rm_b$ rotates $\dv_\star$ to $\ev_3$, and $\Rm_w$ aligns the third axis with $\fv_r/\|\fv_r\|$. The first axis of $\Rm_w$ is the normalized projection of the heading direction ($\coloneqq \Rm_r \ev_1$) onto the plane orthogonal to $\fv_r$; the second axis completes the right-handed orthonormal frame. Consequently, \mbox{$\Rm_w \ev_3 = \Rm_d \dv_\star = \fv_r/\|\fv_r\|$}.

The attitude controller then computes a reference control moment $\tauv_r \in \mathbb{R}^3$ as:
\begin{equation}
    \tauv_r = \Omegav^\times\Jm\Omegav -\Km_R \ev_R - \Km_\Omega \Omegav,
\end{equation}
where $\ev_R \coloneqq \frac{1}{2}[\Rm_d^\top \Rm - \Rm^\top \Rm_d]_\vee$ is the attitude tracking error and $\Km_R, \Km_{\Omega} \in \mathbb{R}^{3\times3}$ are positive definite matrices.

Finally, the wrench mapper searches for the feasible control input that best reproduces the desired wrench $\left[\begin{smallmatrix} (\Rm \fv_r)^\top & \tauv_r^\top\end{smallmatrix}\right]^\top$. To achieve this, and considering~\eqref{eq:assumpt}, we set:
\begin{subequations}
\begin{align}
    \uv_\tau &= \Cm^{-1}\tauv_r, \label{eq:utau} \\
    u_f &= \arg \min_{\xi \in \mathbb{R}}\|\Rm \Am \xi - (\fv_r - \Rm \Bm \Cm^{\dagger}\tauv_r)\|^2. \label{eq:uf}
\end{align}
\end{subequations}
A closed-form solution is given by $u_f = \Am^\dagger (\Rm^\top \fv_r - \Bm \Cm^{-1}\tauv_r)$, which reduces to
\begin{equation}
    u_f = \frac{\av^\top \Rm^\top}{\|\Rm \av\|^2} (\fv_r - \Rm \Bm \Cm^{\dagger}\tauv_r).
\end{equation}

\section{Stability Analysis}
\label{sect:proof}
This section presents the main theoretical result, namely a Lyapunov-based proof of local exponential stability of the hovering equilibrium for the closed-loop error dynamics under the proposed geometric controller. The analysis explicitly accounts for the presence of spurious forces and derives verifiable conditions ensuring stability. Throughout the proof, without loss of generality, the gain matrices are assumed to be isotropic.

\begin{problem}\label{prob:main}
    Given the system~\eqref{eq:system},  and the control law~\eqref{eq:utau}--\eqref{eq:uf}, prove that the closed-loop system is locally exponential stable  to a suitable set where $\pv =\pv_r$, and $\vv$ and $\Omegav$ are both zero, while the orientation  $\Rm = \Rm_d$. 
\end{problem}

To analyze the stability of the proposed controller, we first derive the closed-loop error dynamics.
By differentiating the position and attitude errors, $\ev_p$ and $\ev_R$, and substituting the system dynamics~\eqref{eq:model} along with the control inputs defined in~\eqref{eq:utau}--\eqref{eq:uf}, we obtain the following closed-loop system:
\begin{equation}\label{eq:CL_stability3}
    \left\{\begin{split}
  \dot{\ev}_p &= \vv,\\[2pt]
  m\dot{\vv} &= -\Km_p \ev_p - \Km_v \vv + \Xm
                 \\&+ \Pm_{\star}^\perp \Rm \Bm \Cm^{-1}
                   \bigl(\Omegav^\times\Jm\Omegav
                   - \Km_R \ev_R - \Km_{\Omegav} \Omegav\bigr),\\[4pt]
  \dot{\ev}_R &=\Cm(\Rm_d^\top \Rm)  (\Omegav - \Rm^\top \Rm_d \Omegav_d)\\
  \Jm\dot{\Omegav} &= -\Km_R \ev_R - \Km_{\Omegav} \Omegav.
\end{split}\right.
\end{equation}
where we introduced the following definitions: $\Cm(\Rm_d^\top \Rm)\coloneqq\frac{1}{2} (\operatorname{Tr}[\Rm^\top \Rm_d]\mathbf{I} -\Rm^\top \Rm_d)$; the term $\Xm \coloneqq -\Pm_\star^\perp \fv_r = \|\fv_r\|(\dv_\star^\top \Rm_d^\top \Rm\dv_\star \Rm\dv_\star - \Rm_d\dv_\star)$, which quantifies the misalignment between the preferential direction and the desired thrust direction; and $\Pm_\star^\perp \coloneqq \mathbf{I}_3 - \Rm_d\dv_\star (\Rm_d\dv_\star)^\top$, which is the orthogonal projector onto the plane perpendicular to the desired thrust direction $\Rm_d\dv_\star$. Finally, the term $\Pm_\star^\perp \Rm \Bm \Cm^{-1}\tauv_r$ represents non-compensable spurious force arising from actuation coupling, and $\Omegav_d$ is defined such that $\Rm^\top_d\dot{\Rm}_d=\Omegav_d^\times$. See Appendix~\ref{appendix:e_R_dot} for details on the derivation of $\dot{\ev}_R$.


\subsection{Desired Angular Velocity}
    In this subsection, we derive a closed-form expression for the desired angular velocity $\Omegav_d$ as a function of the closed-loop error states and establish an explicit upper bound on $\|\Omegav_d\|$. These expressions are crucial for analyzing the closed-loop dynamics and proving stability.
    
    \begin{lem}\label{lem:omega_d}
        The vector $\Omegav_d$ admits the following closed-form expression:
    \begin{equation}
         \Omegav_d
         =\Rm_d^\top\!\left[\bm{I}_3+\frac{\hat{\rv}_1^\top \wv_3}{\|\wv_3\times\hat{\rv}_1\|^2}\,\wv_3\hat{\rv}_1^\top\right]
         (\wv_3\times\dot{\wv}_3),
         \label{eq:omega_d}
    \end{equation}
    where $\wv_3=\hat{\fv}_r$ denotes the third column of the rotation matrix $\Rm_w$, $\hat{\rv}_1 \coloneqq \Rm_r\ev_1$ is the constant heading direction, and \mbox{$\dot{\wv}_3 = \dot{\hat{\fv}}_r = \frac{1}{\|\fv_r\|}\left(\mathbf{I}_3 - \hat{\fv}_r \hat{\fv}_r^\top\right)\dot{\fv}_r$}, with $\dot{\fv}_r = -\Km_p \vv - \Km_v \dot{\vv}$.
    \end{lem}
    
    
    \begin{pf}
        Consider the rotation matrix
        \begin{equation}
            \Rm_w = \begin{bmatrix} \wv_1 & \wv_2 & \wv_3 \end{bmatrix},
        \end{equation}
        where $\wv_3=\hat{\fv}_r$ denotes the desired force direction, and the vectors $\wv_1,\wv_2$ complete a right-handed orthonormal frame together with the constant heading direction $\hat{\rv}_1$, namely
        \begin{equation}
            \wv_2 = \frac{\wv_3 \times \hat{\rv}_1}{\|\wv_3 \times \hat{\rv}_1\|}, 
            \qquad 
            \wv_1 = \wv_2 \times \wv_3.
        \end{equation}
        
        The time derivative of $\Rm_w$ is given by
        \begin{equation}
            \dot{\Rm}_w = \begin{bmatrix} \dot{\wv}_1 & \dot{\wv}_2 & \dot{\wv}_3 \end{bmatrix},
        \end{equation}
        where, since $\hat{\rv}_1$ is constant,
        \begin{align}
            \dot{\wv}_3 &= \dot{\hat{\fv}}_r, \\
            \dot{\wv}_2 &= \frac{1}{\|\wv_3 \times \hat{\rv}_1\|}
            \left(\mathbf{I}_3 - \wv_2 \wv_2^\top\right)
            \left(\dot{\wv}_3 \times \hat{\rv}_1\right), \\
            \dot{\wv}_1 &= \dot{\wv}_2 \times \wv_3 + \wv_2 \times \dot{\wv}_3.
        \end{align}
        Here we used the standard identity for a unit vector $\hat{\kv}$,
        \begin{equation}
            \dot{\hat{\kv}} = \frac{1}{\|\kv\|}
            \left(\mathbf{I}_3 - \hat{\kv}\hat{\kv}^\top\right)\dot{\kv}.
        \end{equation}
        
        Since $\dot{\wv}_i = \omegav_w \times \wv_i$ for $i=1,2,3$, the angular velocity $\omegav_w$ admits the decomposition
        \begin{equation}
            \omegav_w
            = \underbrace{\wv_3 \times \dot{\wv}_3}_{\perp \wv_3}
            + \underbrace{(\wv_3^\top \omegav_w)\wv_3}_{\parallel \wv_3}.
        \end{equation}
        
        To compute the scalar component $\wv_3^\top \omegav_w$, we exploit the identity $\wv_3 = \wv_1 \times \wv_2$, which yields
        \begin{equation}
            \wv_3^\top \omegav_w
            = (\wv_1 \times \wv_2)^\top \omegav_w
            = \wv_2^\top (\omegav_w \times \wv_1)
            = \wv_2^\top \dot{\wv}_1.
        \end{equation}
        Substituting the expression of $\dot{\wv}_1$ and using orthogonality, we obtain
        \begin{align}
            \wv_2^\top \dot{\wv}_1
            &= \wv_2^\top (\dot{\wv}_2 \times \wv_3) \nonumber\\
            &= \frac{1}{\|\wv_3 \times \hat{\rv}_1\|}
            \,\wv_2^\top
            \Bigl[
            \bigl(\mathbf{I}_3 - \wv_2 \wv_2^\top\bigr)
            (\dot{\wv}_3 \times \hat{\rv}_1)
            \Bigr]
            \times \wv_3.
        \end{align}
        Applying the vector triple-product identity, this expression simplifies to
        \begin{equation}
            \wv_2^\top \dot{\wv}_1
            = \frac{(\dot{\wv}_3^\top \wv_3)(\hat{\rv}_1^\top \wv_2)
                  - (\hat{\rv}_1^\top \wv_3)(\dot{\wv}_3^\top \wv_2)}
                   {\|\wv_3 \times \hat{\rv}_1\|}.
        \end{equation}
        Since $\hat{\rv}_1^\top \wv_2 = 0$, it follows that
        \begin{equation}
            \wv_2^\top \dot{\wv}_1
            = -\frac{(\hat{\rv}_1^\top \wv_3)(\dot{\wv}_3^\top \wv_2)}
                   {\|\wv_3 \times \hat{\rv}_1\|}
            = \frac{(\hat{\rv}_1^\top \wv_3)
                   \bigl(\hat{\rv}_1^\top (\wv_3 \times \dot{\wv}_3)\bigr)}
                   {\|\wv_3 \times \hat{\rv}_1\|^2}.
        \end{equation}
        
        Substituting this result into the decomposition of $\omegav_w$ yields
        \begin{equation}\label{eq:omega_d_proof}
            \omegav_w
            = \wv_3 \times \dot{\wv}_3
            + \frac{(\hat{\rv}_1^\top \wv_3)\,
                    \hat{\rv}_1^\top(\wv_3 \times \dot{\wv}_3)}
                   {\|\wv_3 \times \hat{\rv}_1\|^2}\,\wv_3,
        \end{equation}
        from which the expression \eqref{eq:omega_d} directly follows.
    \end{pf}
    
    \begin{lem}\label{lem:omega_d_bound}
        Let $\delta\in(0,1)$ be a constant. For any system configuration such that the geometric condition $|\hat{\rv}_1^\top \wv_3|\le \delta$ holds, the norm of the desired angular velocity is bounded by:
        \begin{equation}\label{eq:omega_d_bound}
            \|\Omegav_d\|
            \le \left(1 + \frac{\delta}{1 - \delta^2}\right)\frac{\|\dot{\fv}_r\|}{\|\fv_r\|}.
        \end{equation}
    \end{lem}
    
    \begin{pf}
        The bound \eqref{eq:omega_d_bound} is obtained by starting from the closed-form
        expression \eqref{eq:omega_d} and exploiting the sub-multiplicative property of matrix norms, which yields
        \begin{equation}
            \|\Omegav_d\|
            \le
            \left\|
            \bm{I}_3
            +
            \frac{\hat{\rv}_1^\top \wv_3}{\|\wv_3\times\hat{\rv}_1\|^2}
            \,\wv_3\hat{\rv}_1^\top
            \right\|
            \,
            \|\wv_3 \times \dot{\wv}_3\|.
        \end{equation}
        
        Since $|\hat{\rv}_1^\top \wv_3|\le \delta<1$, it follows that
        $\|\wv_3\times\hat{\rv}_1\|^2 = 1-(\hat{\rv}_1^\top \wv_3)^2 \ge 1-\delta^2$, and therefore
        \begin{equation}
            \left\|
            \bm{I}_3
            +
            \frac{\hat{\rv}_1^\top \wv_3}{\|\wv_3\times\hat{\rv}_1\|^2}
            \,\wv_3\hat{\rv}_1^\top
            \right\|
            \le
            1+\frac{\delta}{1-\delta^2}.
        \end{equation}
        
        Moreover, since $\|\wv_3\|=1$, the cross-product term satisfies
        $\|\wv_3\times\dot{\wv}_3\|\le\|\dot{\wv}_3\|$.
        Using the definition of $\dot{\wv}_3$, we obtain
        \begin{equation}
            \|\dot{\wv}_3\|
            =
            \frac{1}{\|\fv_r\|}
            \|(\mathbf{I}_3-\hat{\fv}_r\hat{\fv}_r^\top)\dot{\fv}_r\|
            \le
            \frac{\|\dot{\fv}_r\|}{\|\fv_r\|}.
        \end{equation}
        Combining the above inequalities yields \eqref{eq:omega_d_bound}.
    \end{pf}

\subsection{Equilibrium Point}
    \begin{prop}\label{thm:eqpoint} (Equilibrium of the Closed-Loop System)
         Consider the closed-loop error system~\eqref{eq:CL_stability3}. The origin $\xv^\circ=\mathbf{0}$ is an equilibrium point.
    \end{prop}
    \begin{pf}
        For the first and the fourth equations, it is immediate that the right-hand side vanishes at equilibrium: $\dot{\ev}_p=\vv=\bm{0}$, and the moment feedback terms are proportional to $\ev_R$ and $\Omegav$, hence $\dot{\Omegav}=\bm{0}$ when $\ev_R=\bm{0}$ and $\Omegav=\bm{0}$.
        
        For the second equation, substituting $\ev_p = \vv = \ev_R = \Omegav = \bm{0}$ yields
        \begin{equation}
            m\dot{\vv} = -\Km_p \bm{0} - \Km_v \bm{0} + \Xm + \Pm_\star^\perp \Rm \Bm \Cm^{-1} \bm{0} = \Xm.
        \end{equation}
        
        At equilibrium, the reference force is $\fv_r = mg \ev_3$, and the desired orientation satisfies $\Rm_d \dv^\star = \ev_3$, which implies $\Xm = \fv_r - \fv_r = \bm{0}$. Hence, $\dot{\vv} = \bm{0}$.
        
        Finally, substituting \eqref{eq:omega_d} into the third equation and using
        $\vv=\dot{\vv}=\bm 0 \Rightarrow \dot{\fv}_r=\bm 0 \Rightarrow \dot{\wv}_3=\bm 0$, we get that $\dot{\ev}_R = \bm{0}$, which completes the proof that $\xv^\circ$ is an equilibrium point of the closed-loop system.
    \end{pf}

\subsection{Candidate Lyapunov Function}
    \begin{defn}\label{def:Lpsi_domain}
        Fix $\psi\in(0,1)$. We define the set
        \begin{equation}\label{eq:Lpsi_def}
            \mathcal{L}_\psi := \left\{ (\Rm, \Rm_d) \in \SO \times \SO \mid \Psi(\Rm, \Rm_d) \le \psi \right\}.
        \end{equation}
    \end{defn}
    The set $\mathcal{L}_\psi$ is a compact sublevel set of the attitude error function and is strictly contained in the domain $\mathcal{L}_1$, representing the set of admissible attitude configurations for which the eigen-axis rotation angle between $\Rm$ and $\Rm_d$ is strictly less than $\pi/2$. Within $\mathcal{L}_1$, the mapping from the configuration to the attitude error vector $\ev_R$ is bijective and non-singular.
    
    Within $\mathcal{L}_\psi$, the attitude error satisfies the following quadratic bounds:
    \begin{equation} \label{eq:psi_quad_bounds}
        \frac{1}{2}\|\ev_R\|^2 \le \Psi(\Rm, \Rm_d) \le \frac{1}{2-\psi}\|\ev_R\|^2.
    \end{equation}
    This property allows us to derive the explicit upper bound on the error norm used throughout the stability analysis:
    \begin{equation} \label{eq:er_max}
        \|\ev_R\| \le \sqrt{\psi(2-\psi)} =: e_R^{\max}.
    \end{equation}
    
    
    
    Given the attitude error bound $e_R^{\max}$ derived in \eqref{eq:er_max}, and selecting positive constants $e_p^{\max}$, $v^{\max}$, and $\Omega^{\max}$ for the position, translational velocity, and angular velocity errors, we define the operational domain $\mathcal{D}$ for our analysis:
    \begin{equation}\label{eq:region_stability}
        \resizebox{\columnwidth}{!}{$
        \mathcal{D} := \left\{
        \begin{aligned}
        (\ev_p, \vv, \ev_R, \Omegav) \in \mathbb{R}^{12}
        \;\mid\;& \|\ev_p\| \le e_p^{\max}, \|\vv\| \le v^{\max},\\
               & \|\ev_R\| \le e_R^{\max}, \|\Omegav\| \le \Omega^{\max}
        \end{aligned}
        \right\}
        $}
    \end{equation}
    The bounds $e_p^{\max}$ and $v^{\max}$ are chosen sufficiently small such that, for a given design parameter $\delta \in (0, 1)$, the geometric condition $|\hat{\rv}_1^\top \wv_3| \le \delta$ is strictly satisfied for all configurations in $\mathcal{D}$. Moreover, recalling the definition of the reference control force in \eqref{eq:f_r}, the bounds $\|\ev_p\|\le e_p^{\max}$ and $\|\vv\|\le v^{\max}$ imply the uniform lower bound
    \begin{equation}\label{eq:fr_lower_bound}
        \|\fv_r\|
        \ge
        mg - k_p e_p^{\max} - k_v v^{\max}
        =: \underline f.
    \end{equation}
    We select $e_p^{\max}$ and $v^{\max}$ such that $\underline f>0$.
    
    All subsequent estimates and inequalities are evaluated over $\mathcal{D}$. Let the Lyapunov candidate $V : \mathcal{D} \to \mathbb{R}$ be defined as
    \begin{equation}\label{eq:V}
        V(\ev_p,\vv,\ev_R,\Omegav) \coloneqq V_1(\ev_p,\vv) + V_2(\ev_R,\Omegav),
    \end{equation}
    with
    \begin{subequations}
        \begin{align}
            V_1(\ev_p,\vv) &= \tfrac12 m \|\vv\|^2 
                           + \tfrac12 k_p \|\ev_p\|^2 
                           + c_1\, \ev_p^\top \vv, \label{eq:V1}\\[3pt]
            V_2(\ev_R,\Omegav) &= \tfrac12 \Omegav^\top \Jm \Omegav 
                                + k_R\, \Psi(\Rm,\Rm_d) 
                                + c_2\, \ev_R^\top \Omegav. \label{eq:V2}
        \end{align}
    \end{subequations}
    Here $c_1, c_2 \ge 0$ are scalar design parameters to be selected later.
    
    For convenience, define the normed error stacks
    $\zv_1 \coloneqq [\,\|\ev_p\| \;\; \|\vv\|\,]^\top$ and
    $\zv_2 \coloneqq [\,\|\ev_R\| \;\; \|\Omegav\|\,]^\top$.
    We first recall some useful lemmas that establish bounds on $V_1$, $V_2$ and their time derivatives.
    
    \begin{lem}\label{lem:v1bound}
        The function $V_1$ in \eqref{eq:V1} satisfies the following bounds:
        \begin{subequations}
        \begin{align}
            & \frac{1}{2}\zv_1^\top \Mm_{11}\zv_1 \le V_1 \le \frac{1}{2}\zv_1^\top \Mm_{12}\zv_1, \label{eq:V1_bound}
        \end{align}      
        \end{subequations}
    where the matrices $\Mm_{11}$ and $\Mm_{12}$ are defined as
    \begin{subequations}
    \begin{align}
    \Mm_{11}&=\begin{bmatrix}
            k_p & -c_1\\
            -c_1 & m
        \end{bmatrix},\;\;
    \Mm_{12}=\begin{bmatrix}
            k_p & c_1\\
            c_1 & m
        \end{bmatrix}. 
    \label{eq:M_translational}
    \end{align}
    \end{subequations}
    \end{lem}
    
    \begin{pf}
        The result follows by bounding the cross term via Cauchy-Schwarz:
        $-\|\ev_p\|\,\|\vv\| \le \ev_p^\top \vv \le \|\ev_p\|\,\|\vv\|$.
        Substituting these bounds into \eqref{eq:V1} and collecting coefficients of $\|\ev_p\|$ and $\|\vv\|$ yields exactly \eqref{eq:V1_bound}.
    \end{pf}
    
    \begin{lem}\label{lem:v2bound} 
        The function $V_2$ in \eqref{eq:V2} satisfies the following inequalities:
        \begin{subequations}
            \begin{align}
               & \frac{1}{2}\zv_2^\top \Mm_{21}\zv_2 \le V_2 \le \frac{1}{2}\zv_2^\top \Mm_{22}\zv_2,  \label{eq:V2_bound}
            \end{align}
        \end{subequations}
    where the matrices  $\Mm_{21}$ and $\Mm_{22}$ are defined as
    \begin{subequations}
        \begin{align}
            \Mm_{21} &= \begin{bmatrix}
                k_R & -c_2\\
                -c_2 & \lambda_{\min}(\Jm)
            \end{bmatrix},\;
            \Mm_{22} = \begin{bmatrix}
                \frac{2k_R}{2-\psi} & c_2\\
                c_2 & \lambda_{\max}(\Jm)
            \end{bmatrix}.\label{eq:M_attitude}
        \end{align}
    \end{subequations}
    \end{lem}
    \begin{pf}
        Using Cauchy-Schwarz and the spectral bounds of $\Jm$,
        \begin{align}
            \tfrac12\,\lambda_{\min}(\Jm)\,\|\Omegav\|^2
            \;\le\; \tfrac12\,\Omegav^\top \Jm \Omegav
            \;\le\; \tfrac12\,\lambda_{\max}(\Jm)\,\|\Omegav\|^2,\\
            -\,c_2\,\|\ev_R\|\,\|\Omegav\|
            \;\le\; c_2\,\ev_R^\top \Omegav
            \;\le\; c_2\,\|\ev_R\|\,\|\Omegav\|.
        \end{align}
        Combining these with \eqref{eq:psi_quad_bounds} in \eqref{eq:V2} and grouping the terms in $\|\ev_R\|$ and $\|\Omegav\|$ gives \eqref{eq:V2_bound}.
    \end{pf}
    
    \begin{lem}\label{lem:vdot1bound}
        The time derivative of the function $V_1$ in \eqref{eq:V1} satisfies the following inequality:
        \begin{equation}\label{eq:V1_dot_bound}  
        \dot{V}_1 \le -\zv_1^\top \Wm_{1} \zv_1 + \zv_1^\top \Wm_{12} \zv_2,      
        \end{equation}
        where the matrices  $\Wm_{1}$ and $\Wm_{12}$ are defined as
        \begin{subequations}
            \begin{align}
                \Wm_{1}\!&=\!\begin{bmatrix}
                    c_1\frac{k_p}{m}(1 - e_R^{\max}) & -c_1\frac{k_v}{2m}(1 + e_R^{\max})\\
                    -c_1\frac{k_v}{2m}(1 + e_R^{\max}) & k_v(1 - e_R^{\max}) - c_1
                \end{bmatrix},\label{eq:W1}\\
                \Wm_{12}\!&=\!\begin{bmatrix}
                    c_1(g\!+\!k_R\frac{\gamma}{m}) & c_1 \frac{\gamma}{m}\bigl(\lambda_{\max}(\Jm)\Omega^{\max}\!+\!k_{\Omega}\bigr) \\
                    k_p e_p^{\max}\!+\!mg\!+\!\gamma k_R & \gamma\bigl(\lambda_{\max}(\Jm)\Omega^{\max}\!+\! k_{\Omega}\bigr)
                \end{bmatrix},\label{eq:W12}
            \end{align}
        \end{subequations}
        where $
            \gamma \coloneqq \sigma_{\max}(\Bm)/\sigma_{\min}(\Cm).$
    \end{lem}
    \begin{pf}
        The time derivative of $V_1$ along the closed-loop trajectories is given by:
       \begin{align}\label{eq:V1_dot}
            \dot{V}_1 &= -(k_v - c_1)\|\vv\|^2 - c_1 \frac{k_p}{m}\|\ev_p\|^2 - c_1 \frac{k_v}{m}\,\ev_p^\top \vv \notag \\
            &\quad + \left(\tfrac{c_1}{m}\ev_p^\top + \vv^\top\right)\Pm_\star^\perp\Rm \Bm \Cm^{-1} \left(\Omegav^\times\Jm - k_{\Omega}\mathbf{I}_3\right)\Omegav \notag\\
            &\quad - k_R \left(\tfrac{c_1}{m}\ev_p^\top + \vv^\top\right)\Pm_\star^\perp\Rm \Bm \Cm^{-1} \ev_R\notag \\
            &\quad+ \left(\tfrac{c_1}{m}\ev_p^\top + \vv^\top\right)\Xm.
        \end{align}
        
        In order to derive an upper bound, we bound each term by
        repeatedly applying the Cauchy-Schwarz inequality, the triangle inequality,
        and the sub-multiplicative property of induced matrix norms.
    
        In particular, we use the bounds
        \begin{gather}
        \bigl\|\Omegav^\times\Jm - k_{\Omega}\mathbf{I}_3\bigr\|
        \;\le\; \lambda_{\max}(\Jm)\,\Omega^{\max} + k_{\Omega}, \label{eq:u11}\\
        \bigl\|\Pm_\star^\perp\Rm \Bm \Cm^{-1}\bigr\|
        \;\le\; \gamma. \label{eq:u12}
        \end{gather}
    
        From the definition of $\Xm$, it holds
        \begin{equation}
            \|\Xm\|
            \le
            \|\fv_r\|\,
            \Bigl\|
            (\dv_\star^\top \Rm_d^\top \Rm \dv_\star)\,\Rm \dv_\star
            - \Rm_d \dv_\star
            \Bigr\|.
        \end{equation}
        The first factor satisfies
        \begin{equation}
            \|\fv_r\|
            \le
            k_p\|\ev_p\| + k_v\|\vv\| + mg.
        \end{equation}
        To bound the second factor, define $\hat{\fv}_c=\Rm\dv_\star$ and
        $\hat{\fv}_r=\Rm_d\dv_\star$, which are unit vectors. Then,
        \begin{equation}
        (\dv_\star^\top \Rm_d^\top \Rm \dv_\star)\,\Rm \dv_\star - \Rm_d \dv_\star
        =
        \hat{\fv}_c \times (\hat{\fv}_c \times \hat{\fv}_r).
        \end{equation}
        The norm of this vector equals $\sin\theta$, where $\theta$ is the angle between
        $\hat{\fv}_c$ and $\hat{\fv}_r$. Using the relation \eqref{eq:relation_er_psi},
        we obtain
        \begin{equation}
        \Bigl\|
        (\dv_\star^\top \Rm_d^\top \Rm \dv_\star)\,\Rm \dv_\star
        - \Rm_d \dv_\star
        \Bigr\|
        =
        \|\ev_R\|,
        \end{equation}
        which proves 
        \begin{equation}
            \|\Xm\| \;\le\; \bigl(k_p\|\ev_p\| + k_v\|\vv\| + mg\bigr)\|\ev_R\|. \label{eq:u13}
        \end{equation}
    
        Substituting \eqref{eq:u11}, \eqref{eq:u12} and \eqref{eq:u13} together with \eqref{eq:er_max} into \eqref{eq:V1_dot} and collecting like terms yields
        \begin{align}
            \dot V_1 \le\;&
            -(k_v-c_1)\|\vv\|^2
            - \frac{c_1 k_p}{m}\|\ev_p\|^2
            + \frac{c_1 k_v}{m}\|\ev_p\|\,\|\vv\|
            \nonumber\\
            &+ \Bigl(\frac{c_1}{m}\|\ev_p\|+\|\vv\|\Bigr)\gamma
            \Bigl[
            (\lambda_{\max}(\Jm)\Omega^{\max}+k_\Omega)\|\Omegav\|
            \nonumber\\
            &\hspace{4em}
            + k_R\|\ev_R\|
            \Bigr]
            \nonumber\\
            &+ \Bigl(\frac{c_1}{m}\|\ev_p\|+\|\vv\|\Bigr)
            \bigl(k_p\|\ev_p\|+k_v\|\vv\|+mg\bigr)\|\ev_R\|.
        \end{align}
        To eliminate higher-order mixed terms, we enforce the local bound
        $\|\ev_p\|\le e_p^{\max}$ and again use \eqref{eq:er_max} to absorb attitude-error factors. This yields
        \begin{align}
            \dot V_1 \le\;&
            -\bigl(k_v(1-e_R^{\max})-c_1\bigr)\|\vv\|^2\nonumber
            - \frac{c_1 k_p}{m}(1-e_R^{\max})\|\ev_p\|^2
            \nonumber\\
            &+ \frac{c_1 k_v}{m}(1+e_R^{\max})\|\ev_p\|\,\|\vv\| \nonumber\\
            &+ c_1\Bigl(g+\gamma\frac{k_R}{m}\Bigr)\|\ev_p\|\,\|\ev_R\|
            \nonumber\\
            &+ \bigl(k_p e_p^{\max}+mg+\gamma k_R\bigr)\|\vv\|\,\|\ev_R\|
            \nonumber\\
            &+ \frac{c_1\gamma}{m}
            (\lambda_{\max}(\Jm)\Omega^{\max}+k_\Omega)\|\ev_p\|\,\|\Omegav\|
            \notag\\
            &+ \gamma(\lambda_{\max}(\Jm)\Omega^{\max}+k_\Omega)\|\vv\|\,\|\Omegav\|.
        \end{align}
    
        Finally, defining $\zv_1=[\|\ev_p\|\;\|\vv\|]^\top$ and
        $\zv_2=[\|\ev_R\|\;\|\Omegav\|]^\top$, regrouping homogeneous terms leads to the
        compact quadratic bound \eqref{eq:V1_dot_bound}.
    \end{pf}
    
    \begin{lem}\label{lem:vdot2bound}
        The time derivative of the function $V_2$ in \eqref{eq:V2} satisfies
        \begin{equation}
            \dot{V}_2 \le -\,\zv_2^\top \Wm_2 \zv_2 \;+\; \zv_1^\top \Wm_{21}\zv_2.
            \label{eq:V2_dot_bound}
        \end{equation}
        The matrices $\Wm_{2}$ and $\Wm_{21}$ are given by
        \begin{subequations}\label{eq:W_matrices}
        \begin{align}
            \resizebox{\columnwidth}{!}{$
            \Wm_2 =
            \begin{bmatrix}
            c_2 \dfrac{k_R}{\lambda_{\max}(\Jm)} - k_R k_v \alpha (mg + k_R\gamma)
            & -\dfrac{\sigma_{R,\Omega}}{2}\\[2pt]
            -\dfrac{\sigma_{R,\Omega}}{2}
            & k_{\Omega} - c_2\!\left(1 + k_v\gamma\alpha(\beta + k_{\Omega})\right)
            \end{bmatrix}
            $}\label{eq:W2}\\[2pt]
            \resizebox{\columnwidth}{!}{$
            \Wm_{21} =
            \begin{bmatrix}
            k_R k_p k_v \alpha (1 + e_R^{\max})
            & c_2 k_p k_v \alpha (1 + e_R^{\max})\\[2pt]
            k_R \alpha m \!\left(\left|\dfrac{k_v^2}{m} - k_p\right|
            + \dfrac{k_v^2}{m}e_R^{\max}\right)
            & c_2 \alpha m \!\left(\left|\dfrac{k_v^2}{m} - k_p\right|
            + \dfrac{k_v^2}{m}e_R^{\max}\right)
            \end{bmatrix}
            $}\label{eq:W21}
        \end{align}
        \end{subequations}
        Here,
        \begin{align}
            \sigma_{R,{\Omega}} &\coloneqq
            c_2\!\left(\frac{k_{\Omega}}{\lambda_{\min}(\Jm)} + k_v \gamma \alpha \bigl(mg + k_R \gamma\bigr)\right)\notag\\
            &\;\quad+ k_R k_v \gamma \alpha \bigl(\beta + k_{\Omega}\bigr),\label{eq:sigma_Romega}\\
            \alpha &\coloneqq \left(1 + \frac{\delta}{1 - \delta^2}\right)\frac{1}{m\,\underline f},\quad \beta \coloneqq \lambda_{\max}(\Jm)\,\Omega^{\max}.\label{eq:alpha_beta}
        \end{align}
    \end{lem}

    \begin{pf}
        The time derivative of $V_2$ along the closed-loop trajectories is given by
        \begin{align}\label{eq:V2_dot}
            \dot{V}_2 = &-k_\Omega \|\Omegav\|^2 + c_2 \Omegav^\top \Cm(\Rm_d^\top \Rm)\Omegav - c_2 k_R \ev_R^\top \Jm^{-1} \ev_R \notag \\
            &- c_2 k_\Omega \ev_R^\top \Jm^{-1} \Omegav  - c_2 \Omegav^\top \Cm(\Rm_d^\top \Rm)\Rm^\top \Rm_d \Omegav_d \notag \\
            &- k_R \ev_R^\top \Rm^\top \Rm_d \Omegav_d.
        \end{align} 
        To derive an upper bound, we bound each term using the Cauchy-Schwarz inequality, orthogonality of rotation matrices, positive definiteness of $\Jm$, and the following relations 
        \begin{gather}
            \|\Cm(\Rm_d^\top \Rm)\| \;\le\; 1,\label{eq:u21}\\
            -c_2k_R\,\ev_R^\top \Jm^{-1}\ev_R \;\le\; -\frac{c_2k_R}{\lambda_{\max}(\Jm)}\|\ev_R\|^2,\label{eq:u22}\\
            -c_2k_\Omega\,\ev_R^\top \Jm^{-1}\Omegav \;\le\; \frac{c_2k_\Omega}{\lambda_{\min}(\Jm)}\|\ev_R\|\,\|\Omegav\|,\label{eq:u23}
        \end{gather}
        where \eqref{eq:u21} follows from \eqref{eq:bound_C}.
    
        Collecting these estimates yields
        \begin{align}
            \dot{V}_2 \le\;&
            -(k_\Omega - c_2)\|\Omegav\|^2
            - c_2\,\frac{k_R}{\lambda_{\max}(\Jm)}\|\ev_R\|^2 \nonumber\\
            &+ c_2\,\frac{k_\Omega}{\lambda_{\min}(\Jm)}\|\ev_R\|\,\|\Omegav\|\nonumber\\
            &+ c_2\|\Omegav\|\,\|\Omegav_d\| + k_R\|\ev_R\|\,\|\Omegav_d\|.
            \label{eq:V2_dot_intermediate}
        \end{align}
        To further bound the terms involving $\Omegav_d$, we refine the estimate \eqref{eq:omega_d_bound} by making explicit the dependence of $\dot{\fv}_r$ on the translational errors, the auxiliary term $\Xm$, and the reference torque $\tauv_r$. In particular,
        \begin{align}
            \dot{\fv}_r
            &= \frac{1}{m}\Km_v\Km_p \ev_p + \Bigl(\frac{1}{m}\Km_v^{2} - \Km_p\Bigr)\vv\nonumber\\
            &- \frac{1}{m}\Km_v \Xm - \frac{1}{m}\Km_v \Pm_\star^\perp \Rm \Bm \Cm^{-1} \tauv_r.
        \end{align}
        Using the triangle inequality and submultiplicativity of norms, together with \eqref{eq:u11}, we obtain
        \begin{equation}\label{eq:fdot_bound}
            \|\dot{\fv}_r\|
            \le
            \frac{k_v k_p}{m}\|\ev_p\|
            + \Bigl|\frac{k_v^{2}}{m} - k_p\Bigr|\|\vv\|
            + \frac{k_v}{m}\|\Xm\|
            + \frac{k_v}{m}\gamma\|\tauv_r\|.
        \end{equation}
        Substituting \eqref{eq:fdot_bound}, \eqref{eq:u12}, and \eqref{eq:u13}, and using \eqref{eq:fr_lower_bound} in \eqref{eq:omega_d_bound} yields the bound 
        \begin{equation}\label{eq:u26}
            \begin{aligned}
            \|\Omegav_d\|
            \le\;&
            \left(1 + \frac{\delta}{1 - \delta^2}\right)
            \Biggl[
            \frac{k_p k_v}{m\,\underline f}(1 + e_R^{\max})\,\|\ev_p\| \\
            &+ \frac{1}{\underline f}
            \left(\left|\frac{k_v^2}{m} - k_p\right|
                  + \frac{k_v^2}{m} e_R^{\max}\right)\|\vv\| \\
            &+ \frac{k_v}{\underline f}\left(g + \frac{\gamma k_R}{m}\right)\|\ev_R\| \\
            &+ \frac{k_v}{m\,\underline f}\gamma
            \bigl(\lambda_{\max}(\Jm)\Omega^{\max} + k_\Omega\bigr)\|\Omegav\|
            \Biggr].
            \end{aligned}
        \end{equation}
        Substituting \eqref{eq:u26} into \eqref{eq:V2_dot_intermediate} and regrouping quadratic and cross terms yields
        \begin{equation}\label{eq:V2_dot_long}
        \resizebox{\columnwidth}{!}{$
        \begin{aligned}
            \dot V_2 \le\;& -\Bigl[k_\Omega\! -\! c_2\Bigl(1\!+\!\Bigl(1\!+\!\frac{\delta}{1\!-\!\delta^2}\Bigr)\frac{k_v}{m\underline f}\gamma(\lambda_{\max}(\Jm)\Omega^{\max}\!+\!k_\Omega)\Bigr)\Bigr]\|\Omegav\|^2 \\
            & + \Bigl[\Bigl(1 + \frac{\delta}{1\!-\!\delta^2}\Bigr)\frac{k_R k_v}{\underline f}\Bigl(g + \frac{k_R \gamma}{m}\Bigr) - c_2\frac{k_R}{\lambda_{\max}(\Jm)}\Bigr]\|\ev_R\|^2 \\
            & + \Biggl\{ c_2\Bigl[\frac{k_\Omega}{\lambda_{\min}(\Jm)} + \Bigl(1 + \frac{\delta}{1\!-\!\delta^2}\Bigr)\frac{k_v}{\underline f}\Bigl(g + \frac{k_R \gamma}{m}\Bigr)\Bigr] \\
            & \qquad + k_R\Bigl(1 + \frac{\delta}{1\!-\!\delta^2}\Bigr)\frac{k_v}{m\underline f}\gamma(\lambda_{\max}(\Jm)\Omega^{\max} + k_\Omega) \Biggr\}\|\ev_R\|\,\|\Omegav\| \\
            & + \Bigl(1 + \frac{\delta}{1\!-\!\delta^2}\Bigr)\frac{k_R k_p k_v}{m\underline f}(1 + e_R^{\max})\|\ev_R\|\,\|\ev_p\| \\
            & + \Bigl(1 + \frac{\delta}{1\!-\!\delta^2}\Bigr)\frac{k_R}{\underline f}\Bigl(\Bigl|\frac{k_v^2}{m} - k_p\Bigr| + \frac{k_v^2}{m}e_R^{\max}\Bigr)\|\ev_R\|\,\|\vv\| \\
            & + c_2\Bigl(1 + \frac{\delta}{1\!-\!\delta^2}\Bigr)\frac{k_p k_v}{m\underline f}(1 + e_R^{\max})\|\Omegav\|\,\|\ev_p\| \\
            & + c_2\Bigl(1 + \frac{\delta}{1\!-\!\delta^2}\Bigr)\frac{1}{\underline f}\Bigl(\Bigl|\frac{k_v^2}{m} - k_p\Bigr| + \frac{k_v^2}{m}e_R^{\max}\Bigr)\|\Omegav\|\,\|\vv\|.
        \end{aligned}
        $}
        \end{equation}
        where $\zv_1=[\|\ev_p\|\;\|\vv\|]^\top$.
        Using the shorthands \eqref{eq:alpha_beta} and \eqref{eq:sigma_Romega}, and
        collecting homogeneous terms, we finally obtain the compact quadratic bound
        \eqref{eq:V2_dot_bound}.
    \end{pf}

\subsection{Stability of the Error System}
    The following theorem establishes the local exponential stability of the equilibrium $\xv^\circ = \bm{0}$, defined in Problem~\ref{prob:main}, under the control law~\eqref{eq:utau}--\eqref{eq:uf}. The proof leverages the properties derived in Lemmas~\ref{lem:v1bound}-\ref{lem:vdot2bound}.

    \begin{thm}\label{thm:stability}
    Consider the closed-loop system in \eqref{eq:CL_stability3}. 
    Let $\mathcal{D}\subset\mathbb{R}^{12}$ be the compact analysis domain defined in \eqref{eq:region_stability}, and let $V:\mathcal{D}\to\mathbb{R}$ be the Lyapunov candidate defined in \eqref{eq:V}.  
    
    If the following conditions hold:
    \begin{enumerate}
        \item The control gains $k_p, k_v, k_R, k_{\Omega}$ and the scalar design parameters $c_1,c_2$ satisfy:
        \begin{subequations}\small
                \begin{align}
                    c_1 &< \min \Biggl\{\sqrt{m k_p},\;k_v(1 - e_R^{\max}),\nonumber\\[-1ex]
                    &\hspace{1.5cm}\frac{4 k_p k_v m (e_R^{\max}-1)^2}{k_v^2 (e_R^{\max}+1)^2 + 4 k_p m (1 - e_R^{\max})}\Biggr\}, \label{eq:c1}\\[4pt]
                    c_2 &< \min \Bigl\{\sqrt{\lambda_{\min}(\Jm) k_R},\;\tfrac{k_{\Omega}}{1 + k_v \gamma \alpha (\beta + k_{\Omega})},\;c_{2,+}\Bigr\}, \label{eq:c2_max}\\[4pt]
                    c_2 &> \max \Bigl\{k_v \lambda_{\max}(\Jm)\alpha(mg + k_R \gamma),\;c_{2,-}\Bigr\}, \label{eq:c2_min}
                \end{align}
            \end{subequations}
        where $c_{2,-}$ and $c_{2,+}$ are the two positive solutions of the quadratic inequality in $c_2$ obtained from enforcing the positive definiteness of $\Wm_2$, and where $\alpha$ and $\beta$ are defined in \eqref{eq:alpha_beta}.
        \item The matrices $\Wm_{12}, \Wm_{21}, \Wm_1$, and $\Wm_2$ as in \eqref{eq:W12}, \eqref{eq:W21}, \eqref{eq:W1} and \eqref{eq:W2} respectively satisfy: 
        \begin{equation}
            \|\Wm_{12}+\Wm_{21}\|^2 < 4\lambda_{\min}(\Wm_1)\lambda_{\min}(\Wm_2).
            \label{eq:W_condition}
        \end{equation}
    \end{enumerate}
    
    Then, there exists a constant $c>0$ such that the sublevel set
    \begin{equation}
    \mathcal{D}_0 := \{\xv\in\mathcal{D}\mid V(\xv)\le c\}
    \label{eq:D0_def}
    \end{equation}
    is forward invariant and satisfies $\mathcal{D}_0 \subset \mathcal{D}$.
    
    Consequently, the equilibrium $\xv^\circ$ is locally exponentially stable and $\mathcal{D}_0$ represents an estimate of the region of attraction.
    \end{thm}
    \begin{pf}
        We establish local exponential stability by analyzing the candidate Lyapunov function $V$ and its time derivative along the closed-loop trajectories.
        From Lemmas~\ref{lem:v1bound} and~\ref{lem:v2bound}, the Lyapunov function $V$ admits the quadratic bounds
        \begin{equation}
            \frac{1}{2}\zv^\top \Mm_{1}\zv \;\le\; V(\xv) \;\le\; \frac{1}{2}\zv^\top \Mm_{2}\zv,
        \end{equation}
        where  $\Mm_{j}=\mathrm{blkdiag}\!\left(\{\Mm_{kj}\}_{k=1}^{2}\right)$, $j=1,2$.     
        The matrices $\Mm_{1}$ and $\Mm_{2}$ are positive definite provided that
    \begin{equation}\label{eq:V_bounds_conditions}
            \begin{aligned}
                c_1 &< \sqrt{m\,k_p}, \qquad
                c_2 < \sqrt{\lambda_{\min}(\Jm)\,k_R},\\
                c_2 &< \sqrt{\lambda_{\max}(\Jm)\,k_R\,\frac{2}{2-\psi}}.
            \end{aligned}
        \end{equation}
        Next, differentiating $V$ along the trajectories of the closed-loop system~\eqref{eq:CL_stability3} and applying Lemmas~\ref{lem:vdot1bound} and~\ref{lem:vdot2bound}, we obtain
        \begin{equation}
            \dot V(\xv) \;\le\; -\,\zv^\top \Wm \zv,
            \end{equation}
        where
    \begin{equation}\label{eq:Vdot_bound}
            \Wm =
            \begin{bmatrix}
            \Wm_1 & -\tfrac{1}{2}\!\left(\Wm_{12}+\Wm_{21}\right)\\[2pt]
            -\tfrac{1}{2}\!\left(\Wm_{12}+\Wm_{21}\right) & \Wm_2
            \end{bmatrix}.
        \end{equation}
        The matrix $\Wm$ is positive definite if $\Wm_1 \succ 0$, $\Wm_2 \succ 0$, and condition~\eqref{eq:W_condition} holds. 
        The positive definiteness of the diagonal blocks is ensured by the following inequalities:
        \begin{subequations}
            \begin{align}
                c_1 &< k_v(1-e_R^{\max}),\label{eq:c_11}\\
                c_1 &< \frac{4k_p k_v m (e_R^{\max}-1)^2}
                           {k_v^2 (e_R^{\max}+1)^2 + 4k_p m (1-e_R^{\max})},\\[2pt]
                c_2 &> k_v\,\lambda_{\max}(\Jm)\,\alpha\bigl(mg+k_R\gamma\bigr),\\
                c_2 &< \frac{k_{\Omega}}{1+k_v\gamma\alpha(\beta+k_{\Omega})},\\
                \det(\Wm_2) &> 0 \quad \Leftrightarrow \quad c_{2,-}<c_2<c_{2,+}.
                \label{eq:W2_pd}
            \end{align}
        \end{subequations}
        Condition~\eqref{eq:W2_pd} reduces to the scalar quadratic inequality
    \begin{equation}\label{eq:parabola}
            - A\,c_2^{2} + B\,c_2 - C > 0 , 
        \end{equation}
        where the coefficients $A$, $B$, and $C$ are given by
        \begin{align}
            A &= \frac{1}{4}\Biggl(
                  \frac{k_{\Omega}}{\lambda_{\min}(\Jm)}
                  + k_v\gamma\alpha(mg+k_R\gamma)
                 \Biggr)^{\!2}
                 \notag\\
              &
                 + \frac{k_R}{\lambda_{\max}(\Jm)}
                   \Bigl[1+k_v\gamma\alpha(\beta+k_{\Omega})\Bigr],
            \label{eq:Acoeff}\\[3pt]
            B &= k_R k_v \alpha(mg+k_R\gamma)
                 \Bigl[1+k_v\gamma\alpha(\beta+k_{\Omega})\Bigr] + \frac{k_R k_{\Omega}}{\lambda_{\max}(\Jm)}
                 \label{eq:Bcoeff}\\
              &
                 - \!\frac{1}{2}
                   \!\Biggl(
                   \frac{k_{\Omega}}{\lambda_{\min}(\Jm)}
                   \!+\! k_v\gamma\alpha(mg\!+\!k_R\gamma)
                   \Biggr)
                   \Bigl[k_R k_v\gamma\alpha(\beta\!+\!k_{\Omega})\Bigr],
            \notag\\[3pt]
            C &= k_R k_v \alpha
                 \Biggl(
                 \frac{1}{4}k_R k_v \gamma^{2}\alpha(\beta+k_{\Omega})^{2}
                 + k_{\Omega}(mg+k_R\gamma)
                 \Biggr).
            \label{eq:Ccoeff}
        \end{align}
        Since $A>0$, the parabola opens downward, and feasibility requires a positive discriminant, $B^{2}-4AC>0$. 
        Accordingly, all admissible values of $c_2$ lie strictly between the two positive roots $c_{2,-}$ and $c_{2,+}$.
    
        Therefore, selecting any $c_2\in(c_{2,-},\,c_{2,+})$ guarantees $\Wm_2\succ 0$ and, together with \eqref{eq:W_condition}, implies $\Wm\succ 0$. 
        Collecting \eqref{eq:W_condition}, \eqref{eq:V_bounds_conditions} and \eqref{eq:c_11}-\eqref{eq:W2_pd} yields the parameter conditions stated in the theorem.
        
       Since $\dot V \le -\lambda_{\min}(\Wm)\|\zv\|^2$ holds on the domain $\mathcal{D}$, standard Lyapunov arguments ensure local exponential stability of the equilibrium $\xv^\circ$. 
       Moreover, since $\dot V(\xv)\le 0$ holds for all $\xv\in\mathcal D$ (under the above parameter conditions), any sublevel set of $V$ that is contained in $\mathcal D$ is forward invariant.
        By continuity of $V$ and because $\xv^\circ \in \mathrm{int}(\mathcal{D})$, there exists a constant $c>0$ such that the sublevel set $\mathcal{D}_0$ in \eqref{eq:D0_def} is forward invariant and satisfies $\mathcal{D}_0 \subset \mathcal{D}$.
    \end{pf}
    
    \begin{rem}
        We highlight a fundamental difference from standard geometric approaches (e.g., \cite{2010-LeeLeoMcc}), which typically rely on a \emph{cascade argument} where attitude stability is established autonomously. In our setting, this separation principle fails: spurious forces intrinsically couple the dynamics, preventing independent stability characterizations. Consequently, the proof cannot proceed sequentially but requires a strict Lyapunov function for the \emph{composite} system to certify the joint convergence of both subsystems.
    \end{rem}

\section{Conclusions}
\label{sect:conclusions}

This work provides the first rigorous stability certification for the hovering control problem of floating rigid bodies subject to intrinsic force--moment coupling. While the control strategy proposed by Michieletto \emph{et al.}~\cite{2017f-MicRylFra,2018a-MicRylFra} has demonstrated remarkable robustness in experimental studies, a formal theoretical justification was previously lacking. By analyzing a canonical model that abstracts away the specific actuator topology, we established local exponential stability of the hovering equilibrium for a broad class of coupled underactuated systems.

The proposed Lyapunov-based analysis explicitly accounts for the structural challenges introduced by spurious forces, which invalidate standard cascade and skew-symmetry arguments commonly exploited in geometric control of multirotor vehicles. The resulting stability proof yields explicit algebraic conditions on the controller gains, thereby providing verifiable and practically meaningful design guidelines for safe hovering stabilization in the presence of intrinsic translational--rotational coupling.

Future research directions include reducing the conservatism of the derived bounds and exploiting the additional degrees of freedom available in platforms admitting a decoupled plane (D2) to improve transient performance. Another important extension concerns the trajectory tracking problem. While the present analysis is deliberately restricted to static hovering equilibria, the structure of the proposed controller suggests that a Lyapunov-based stability analysis for time-varying reference trajectories may be feasible, albeit at the cost of additional technical complexity. Establishing formal stability guarantees for such tracking objectives in the presence of force--moment coupling represents a natural and promising direction for future work.

 \bibliographystyle{plain}
 \bibliography{Settings/Bib/bibAlias,Settings/Bib/bibAlias_short,Settings/Bib/custom,Settings/Bib/bibAF,Settings/Bib/bibMain,Settings/Bib/bibNew}


    \appendix

    \section{Appendix}\label{appendix}
We collect several identities used throughout the work.

\subsection{Time Derivative of the Attitude Tracking Error}\label{appendix:e_R_dot}
We follow the standard derivation in \cite{2010-LeeLeoMcc}, adapting notation. Starting from
\begin{equation}\label{eq:attitude_error_repeat}
    \ev_R \coloneqq \tfrac12\left[\Rm_d^\top \Rm - \Rm^\top \Rm_d\right]_\vee,
\end{equation}
and exploiting $\tfrac{d}{dt}(\Rm_d^\top \Rm)=\Rm_d^\top \Rm\,\ev_\omega^\times$ together with $\ev_\omega=\Omegav-\Rm^\top\Rm_d\Omegav_d$, we obtain
\begin{align}
    \dot{\ev}_R &= \tfrac12\!\left(\Rm_d^\top \Rm\,\ev_\omega^\times + \ev_\omega^\times \Rm^\top \Rm_d\right)_\vee \nonumber 
                \\&= \tfrac12\left(\operatorname{Tr}[\Rm^\top \Rm_d]\mathbf{I}_3 - \Rm^\top \Rm_d\right)\ev_\omega,
\end{align}
where the identity \( (\Am\xv^\times + \xv^\times \Am^\top)_\vee = (\operatorname{Tr}[\Am]\mathbf I_3 - \Am)\xv \), valid for all \( \Am\in\mathbb R^{3\times 3} \) and \( \xv\in\mathbb R^3 \), is used. Defining
\begin{align}\label{eq:C_repeat}
    \Cm(\Rm_d^\top \Rm) \coloneqq \tfrac12\left(\operatorname{Tr}[\Rm^\top \Rm_d]\mathbf{I}_3 - \Rm^\top \Rm_d\right),
\end{align}
gives the compact expression
\begin{align}\label{eq:e_R_dot_repeat}
    \dot{\ev}_R = \Cm(\Rm_d^\top \Rm)\,\ev_\omega.
\end{align}
We now characterize $\Cm(\Rm_d^\top \Rm)$. Let $\Rm_d^\top \Rm = \exp(\xv^\times)$ with $\theta=\|\xv\|\in[0,\pi)$ and $\hat{\xv}=\xv/\theta$. Using Rodrigues' formula,
\begin{equation}
\resizebox{\columnwidth}{!}{$
\Cm(\exp(\xv^\times)) =
\tfrac12\left(2\cos\theta\,\mathbf{I}_3 + \sin\theta\,\hat{\xv}^\times - (1 - \cos\theta)(\hat{\xv}^\times)^2\right).
$}
\end{equation}
To bound its Euclidean operator norm, consider $\Cm^\top \Cm$, which simplifies to
\begin{equation}
    \Cm^\top \Cm = \cos^2\theta\,\mathbf{I}_3 - \left(\tfrac12 + \tfrac12\cos\theta - \cos^2\theta\right)(\hat{\xv}^\times)^2.
\end{equation}
The eigenvalues are $
    \lambda_1 = \cos^2\theta,$ and \ $\lambda_{2,3} = \tfrac12(1+\cos\theta).$
Hence
\begin{equation}\label{eq:bound_C}
    \|\Cm(\Rm_d^\top \Rm)\| = \sqrt{\lambda_{\max}(\Cm^\top \Cm)} \le 1,\;\;\; \forall\,\Rm_d,\Rm \in \SO.
\end{equation}

\subsection{Maximum attitude error}\label{appendix:eR_psi}
    To derive \eqref{eq:er_max}, let us write $\Rm_d^\top \Rm = \exp(\xv^\times)$. Using $\operatorname{Tr}[\exp(\xv^\times)] = 1+2\cos\theta$, we obtain
    \begin{equation}
        \Psi(\Rm,\Rm_d)=\tfrac12\bigl(3-\operatorname{Tr}[\Rm_d^\top \Rm]\bigr)=1-\cos\theta.
    \end{equation}
    By the definition of $\ev_R$, and standard $\SO$ identities,
    \begin{equation}\label{eq:relation_er_psi}
        \|\ev_R\|^2 = \sin^2\theta = (1-\cos\theta)(1+\cos\theta)=\Psi(2-\Psi),
    \end{equation}
    hence
    \begin{equation}\label{eq:e_R_norm}
        \|\ev_R\| = \sqrt{\Psi(\Rm,\Rm_d)\bigl(2-\Psi(\Rm,\Rm_d)\bigr)}.
    \end{equation}
    Therefore, for all $(\Rm,\Rm_d)\in\mathcal L_\psi$ (i.e., whenever $\Psi(\Rm,\Rm_d)\le\psi$), \eqref{eq:e_R_norm} yields the bound \eqref{eq:er_max}.

\end{document}